\documentclass[twoside]{article}

\usepackage[accepted]{AISTATS2025/aistats2025}
\usepackage[utf8]{inputenc}
\usepackage[english]{babel}
\usepackage{amsmath,amssymb,epsfig,amsthm}
\usepackage{graphicx}
\usepackage{comment}
\usepackage{array}
\usepackage{algorithm,algpseudocode}
\usepackage{enumitem}
\usepackage{algpseudocode}
\usepackage{xurl}
\usepackage{xargs}     
\usepackage{subcaption}
\usepackage{mwe}
\usepackage[normalem]{ulem}
\usepackage[pdftex,dvipsnames]{xcolor}  
\usepackage[colorinlistoftodos,prependcaption,textsize=tiny]{todonotes}
\usepackage{appendix}
\usepackage{hyperref}
\usepackage{textalpha}
\usepackage{mathrsfs}

\usepackage[round]{natbib}
\usepackage[margin=1in]{geometry} 
\usepackage[round]{natbib}

\newcommandx{\unsure}[2][1=]{\todo[linecolor=red,backgroundcolor=red!25,bordercolor=red,#1]{#2}}
\newcommandx{\change}[2][1=]{\todo[linecolor=blue,backgroundcolor=blue!25,bordercolor=blue,#1]{#2}}
\newcommandx{\info}[2][1=]{\todo[linecolor=OliveGreen,backgroundcolor=OliveGreen!25,bordercolor=OliveGreen,#1]{#2}}
\newcommandx{\improvement}[2][1=]{\todo[linecolor=Plum,backgroundcolor=Plum!25,bordercolor=Plum,#1]{#2}}
\newcommandx{\thiswillnotshow}[2][1=]{\todo[disable,#1]{#2}}

\allowdisplaybreaks[1]

\newtheorem{theorem}{Theorem}

\newtheorem{assumption}{Assumption}

\newtheorem{corollary}{Corollary}

\newtheorem{lemma}{Lemma}

\newtheorem{proposition}{Proposition}
\newtheorem{remark}{Remark}

\numberwithin{equation}{section}

\newcommand{\calB}{\ensuremath{\mathcal{B}}}

\newcommand{\calD}{\ensuremath{\mathcal{D}}}
\newcommand{\calH}{\ensuremath{\mathcal{H}}}
\newcommand{\calF}{\ensuremath{\mathcal{F}}}

\newcommand{\calN}{\ensuremath{\mathcal{N}}}

\newcommand{\calE}{\ensuremath{\mathcal{E}}}

\newcommand{\norm}[1]{\left\|{#1}\right\|}
\newcommand{\abs}[1]{\left|{#1}\right|}

\newcommand{\expec}{\ensuremath{\mathbb{E}}}

\newcommand{\prob}{\ensuremath{\mathbb{P}}}

\definecolor{asparagus}{rgb}{0.53, 0.66, 0.42}

\newcommand{\indic}{\ensuremath{\mathbf{1}}} 

\newcommand{\R}{\ensuremath{\mathbb{R}}}

\newcommand{\la}{\langle}
\newcommand{\ra}{\rangle}

\newcommand{\wsig}{w_{\text{sig}}}

\newcommand{\wpe}{w_{\perp}}

\newcommand{\upper}{\norm{Qw^*}\|Q^{1/2}w^*\|^{-1}\|Q^{1/2}\|^{-1}_F}
\newcommand{\mS}{\ensuremath{\mathbb{S}}}

\begin{document}

%

%

\twocolumn[

\aistatstitle{Learning a Single Index Model from Anisotropic Data with Vanilla Stochastic Gradient Descent}

\aistatsauthor{ Guillaume Braun \And Ha Quang Minh  \And Masaaki Imaizumi }

\aistatsaddress{ RIKEN AIP \And  RIKEN AIP   \And RIKEN AIP \\ University of Tokyo } ]

\begin{abstract}
We investigate the problem of learning a Single Index Model (SIM)—a popular model for studying the ability of neural networks to learn features—from anisotropic Gaussian inputs by training a neuron using vanilla Stochastic Gradient Descent (SGD). While the isotropic case has been extensively studied, the anisotropic case has received less attention and the impact of the covariance matrix on the learning dynamics remains unclear. For instance, \cite{anisoSIM} proposed a spherical SGD that requires a separate estimation of the data covariance matrix, thereby oversimplifying the influence of covariance. In this study, we analyze the learning dynamics of vanilla SGD under the SIM with anisotropic input data, demonstrating that vanilla SGD automatically adapts to the data's covariance structure. Leveraging these results, we derive upper and lower bounds on the sample complexity using a notion of effective dimension that is determined by the structure of the covariance matrix instead of the input data dimension.
%
Finally, we validate and extend our theoretical findings through numerical simulations, demonstrating the practical effectiveness of our approach in adapting to anisotropic data, which has implications for efficient training of neural networks.

\end{abstract}

\section{Introduction}
In many high-dimensional applications, data sets are often assumed to have an underlying low-dimensional structure (e.g., images or text can be embedded in low-dimensional manifolds). This assumption provides a way to circumvent the curse of dimensionality. 

The Single Index Model (SIM) provides a simple yet powerful statistical framework to evaluate the ability of an algorithm to adapt to the latent dimension of the data. In SIM, the response variable $y \in \R$ is linked to the covariates $x \in \R^d$ through a \emph{link function} $f$ that depends only on a rank-one projection of $x$. More formally, $y=f(\la w^*, x \ra)$, where $w^*\in \R^d$ represents the direction of the latent low-dimensional space and is referred to as the \emph{single-index}. This versatile model generalizes the well-known Generalized Linear Model (GLM) when the link function is known. Additionally, SIM can be extended to capture multiple directions, leading to the Multi-Index Model \citep{AbbeAM23, oko24a_ridge}.

In recent years, SIM has become a popular generative model for studying the ability of neural networks trained with SGD or a variant.
This is because the usage of SIM effectively adapts to the latent dimension of the data, in contrast to the kernel method \citep{mei20, damian22a}. 
In particular, the isotropy of the covariates $x$ plays an important role:
when $x$ are sampled from an isotropic Gaussian distribution, it is well-known that the difficulty of estimating $w^*$ depends on $d$ and the information exponent associated with $f$ \citep{Arous2020OnlineSG, bietti2022learning} in the online setting, or the generative exponent if one can reuse samples \citep{damian2024computationalstatistical, repetita24, lee2024neural}.



However, while anisotropy of input data is common in real-world applications such as classification, only a few works extend beyond the simplifying assumption of isotropic Gaussian data, limiting their applicability to more realistic, complex scenarios. 
To overcome this limitation, several works handle the anisotropy of $x$, for instance, \cite{beyongG24} extends the SIM analysis to inputs generated from an approximately spherically symmetric distribution. \cite{ba2023learning, anisoSIM} consider a setting where the inputs are sampled from a Gaussian distribution with a covariance matrix having a spike aligned with the single-index $w^*$. The first work proposes a layer-wise training method, whereas the second work studies a version of spherical SGD that requires an estimate of the covariance matrix of $x$. To the best of our knowledge, the only algorithm that comes with general theoretical guarantees is the mean-field Langevin dynamic analyzed by \cite{wu24langevin}. Unfortunately, it is impractical due to computational inefficiencies. 

In practice, the structure of the covariance matrix is often unknown, and rather than specialized algorithms, simple and generic methods such as vanilla SGD are typically employed. This observation motivates our central research question:
\begin{center}
    \textit{``Can vanilla SGD learn a SIM model under a general class of covariance structures?''}
\end{center}
Answering this question will not only demonstrate the learnability of non-isotropic covariance structures but also highlight that widely-used, simple algorithms like vanilla SGD can successfully achieve this goal. Moreover, this work represents a first step toward extending the analysis to more complex input data, such as Gaussian mixtures or functional data \citep{functionalSIM24}.


In this study, we analyze the vanilla SGD to learn a SIM from general anisotropic inputs and show its adaptability to general covariance $Q$.  This contrasts with the result of \cite{anisoSIM} that shows spherical SDG fails if the algorithm is not modified to incorporate information about $Q$. Specifically, we show that our estimator has a constant correlation with the single-index $w^*$ after $T$ SGD iterations and characterize $T$ as a function of $Q$, the alignment of $Q$ with $w^*$ and the information exponent of the link function $f$. Interestingly, our bound depends on an effective dimension determined by $Q$, instead of the input data dimension $d$. We also establish a Correlated Statistical Query (CSQ) lower bound, suggesting that our effective measure of the dimension is correct on average over $w^*$. We illustrate and complement our theoretical findings through numerical simulations. 

One of the main technical challenges is that, unlike spherical SGD, the evolution of the correlation with $w^*$ also depends on the evolution of the norm of weights. We tackle this problem by developing a method that simultaneously controls the evolution of all the parameters. Our analysis provides insight into the interplay between correlation and weight evolution, which could be instrumental in understanding the training dynamics of wider neural networks.


\subsection{Related work}

\paragraph{Single Index Model.} As an extension of the Generalized Linear Model (GLM) \citep{Neld:Wedd:1972}, the Single Index Model (SIM) is a versatile and widely used statistical framework. It has been applied in various domains, including longitudinal data analysis \citep{SIMlong}, quantile regression \citep{SIMquantile}, and econometrics \citep{SIMsmooth}. In recent years, SIMs have attracted increasing attention from the theoretical deep learning community, particularly as a tool to evaluate the ability of neural networks (NNs) to learn low-dimensional representations, often referred to as features, in contrast to kernel methods \citep{mei20}. This distinction is highlighted by the \emph{lazy regime} \citep{jacotNTK}, where neural networks exhibit kernel-like behavior but cannot learn features.

Several works demonstrate that a Single or Multi-Index Model can be learned from isotropic Gaussian inputs by training a two-layer neural network in a layer-wise manner \citep{damian22a, mousavi-hosseini2023neural, bietti2022learning, bietti2023learning, AbbeAM23, zhou2024doesgradientdescentlearn}. However, fewer studies have addressed the more challenging anisotropic case. The most closely related work is \cite{anisoSIM}, which analyzes a specific covariance structure $Q= \frac{I_d+\kappa \theta \theta^\top}{1+\kappa}$ where $\theta \in \mS^{d-1} $ is correlated with target index $w^*$ and $\kappa$ measures the intensity of the spike in the direction $\theta$. In this setting, they show that spherical Gradient Flow (GF) fails to estimate $w^*$, but a modified algorithm using normalization depending on the covariance matrix $Q$ succeeds. In contrast, our work analyzes vanilla SGD and demonstrates that this simpler algorithm succeeds in learning $w^*$ without any prior estimation of 
$Q$, making it agnostic to the covariance matrix. Also, our analysis sheds light on how the learning rate $\eta$ should be chosen, whereas the GF framework is not directly implementable.



\paragraph{Training dynamic.}  The dynamics of SGD and its variants in training shallow neural networks have been extensively studied \citep{mei_mf18, jacotNTK}. Our analysis builds on the framework introduced by \cite{Arous2020OnlineSG} for spherical SGD in the online setting, which provides insight into the behavior of the estimator’s weights over time. Additionally, our approach is related to the recent work by \cite{glasgow2024sgd}, which studies the simultaneous training of a two-layer neural network to learn the XOR function. Similar to their method, we project the weight vector onto signal and noise components to control their respective growth during training. Furthermore, recent studies on the edge of stability (EoS) phenomenon in quadratic models \cite{catapultSGD24,zhu2024quadratic, ChenBGA24} show that two-layer neural networks trained with a larger learning rate than the prescribed one can generalize well after an unstable training phase.

One of the primary goals of our study is to focus on vanilla SGD, a practical algorithm widely used in real-world scenarios, rather than on specialized theoretical variants like spherical SGD. We believe that this work contributes to a deeper understanding of SGD's behavior in non-convex optimization problems and highlights its effectiveness in solving complex learning tasks involving anisotropic data.

\subsection{Notations}
We use $\norm{\cdot}$ and $\la \cdot, \cdot \ra$ to denote the Euclidean norm and scalar product, respectively. When applied to a matrix, $\norm{\cdot}$ refers to the operator norm. Any positive semi-definite matrix $Q$ induces a scalar product defined by $\la x,y \ra_Q=x^\top Q y$. The Frobenius norm of a matrix $A$ is denoted by $\norm{A}_F$. The $d \times d$ identity matrix is represented by $I_d$. The standard Gaussian measure on $\R$ is denoted by $\gamma$ and the corresponding Hilbert space $L^2(\R, \gamma)$ is referred to as $\calH$. The $(d-1)$-dimensional unit sphere is denoted by $\mS^{d-1}$.
We use the notation $a_n \lesssim b_n$ (or $a_n \gtrsim b_n$) for sequences $(a_n)_{n \geq 1}$ and $(b_n)_{n \geq 1}$ if there exists a constant $C > 0$ such that $a_n \leq C b_n$ (or $a_n \geq C b_n$) for all $n$. If the inequalities hold only for sufficiently large $n$, we write $a_n = O(b_n)$ (or $a_n = \Omega(b_n)$).
\section{Learning a SIM}
\paragraph{Model.} We observe for $i=1 \ldots T$ i.i.d. samples $(x^{(i)},y^{(i)})\in \R^{d}\times [-1,1]$ generated by the following process: the inputs $x^{(i)}\sim \calN(0,Q)$ are generated from a Gaussian distribution with covariance matrix $Q$.
Also, $y^{(i)}$ follows the SIM with given $x^{(i)}$ as
\begin{align}
y^{(i)}=f\left( \left\langle x^{(i)}, \frac{w^*}{\norm{Q^{1/2}w^*}} \right\rangle \right),    
\end{align} 
where $w^*\in \mS^{d-1}$ and $f:\R\to [-1,1]$ is an unknown link function. The normalization factor $\norm{Q^{1/2}w^*}$ is introduced so that $\langle x^{(i)}, \frac{w^*}{\norm{Q^{1/2}w^*}}\rangle \sim \calN(0,1)$ and only affects the definition of $f$. 

\begin{assumption}\label{ass:q}
    We assume that $\norm{Q}=1$ and that $\norm{Q^{1/2}w^*}=c^*\leq 1$ is of constant order.
\end{assumption}

This assumption is satisfied for both \( Q = I \) and a spiked covariance matrix of the form \( Q = (1+\kappa)^{-1}(I+\kappa \theta \theta^\top) \), where \( \theta \in \mS^{d-1} \) is such that \( \langle \theta , w^* \rangle \) is of constant order. While we believe our analysis could be extended to the setting where \( \norm{Q^{1/2}w^*} = o(1) \), the scarcity of input in the direction of \( w^* \) makes this more challenging. In such cases, it would be more efficient to correct the sample by estimating \( Q \). 

\paragraph{Learning method.} Following the approach in \cite{Arous2020OnlineSG}, we employ the correlation loss function \[ L(y, \hat{y})= 1 -y\hat{y}.\]  
Our method involves training a single neuron, defined as $$f_w(x)=\sigma(w^\top x)$$ with $w\in \R^d$ and $\sigma$ denotes the ReLU activation function, using vanilla gradient descent. 
The steps are as follows:
\begin{itemize}
  \setlength{\parskip}{0cm} 
  \setlength{\itemsep}{0cm} 
    \item Sample $w'\sim \calN(0,I_d)$ and initialize $w^{(0)}=r \frac{w'}{\norm{w'}}$ where $r>0$ is a scaling parameter.
    \item Update the weight by SGD: $w^{(t+1)}=w^{(t)}-\eta \nabla_{w^{(t)}}L(y^{(t)}, f_{w^{(t)}}(x^{(t)}))$ where $\eta>0$ is the learning rate.
\end{itemize}

Vanilla gradient descent without constraining the weights provides two key advantages: (i) computational efficiency, as it eliminates the need to estimate the covariance matrix $Q$ and compute terms like $\norm{Q^{1/2}w^*}$, and (ii) simplicity, aligning more closely with standard practices in neural network training. 

\begin{remark} 
Previous works such as \cite{Arous2020OnlineSG, bietti2022learning, AbbeAM23} leverage spherical SGD to control the norm of the weights $w^{(t)}$ at each iteration, simplifying theoretical analysis. However, in the anisotropic setting, spherical SGD must be modified with knowledge of the covariance matrix $Q$ to succeed, as shown by \cite{anisoSIM}.
\end{remark}
\section{Main Results}
Before stating our main results, we will introduce additional notations and discuss the required assumptions.

\subsection{Notation and Assumptions}

First, we assume that $f$ has an information exponent $k^*\geq 1$ and is normalized and bounded.
 \begin{assumption} \label{ass:ie}We assume that the link function $f: \R \to [-1,1]$ is such that $\expec(f)=0$, $\expec (f^2)=1$ and  $f$ has information exponent $k^*$ defined as  \[ k^* := \min \lbrace k\geq 1 : \expec (f H_k) \neq 0 \rbrace\] where $H_k$ is the order $k$ Hermite polynomial (see Section \ref{app:hermite} for more details).     
 \end{assumption}

\begin{remark}
    We assumed for simplicity that $f$ takes value in $[-1,1]$. Our analysis could be extended to the class of functions $f$ growing polynomially, i.e., there exists a constant $C>0$ and an integer $p$ such that for all $x\in \R$, $|f(x)|\leq C(1+|x|^p)$. For example, using Lemma 16 in \cite{anisoSIM}, one could obtain a uniform upper bound on $|y^{(i)}|$ and reduce to the case where $f(x)\in [-1,1]$ for all $x$. We leave the full proof to future work.
\end{remark}

The following measure of correlation is defined in terms of the scalar product induced by $Q$ \begin{equation*}\label{eq:defmt}
    m_t :=  \left\la \frac{Q^{1/2}w^{(t)}}{\norm{Q^{1/2}w^{(t)}}}, \frac{Q^{1/2}w^*}{\norm{Q^{1/2}w^*}} \right\ra.
\end{equation*} 
This is in contrast to the isotropic setting, in which the correlation between the weights $w^{(t)}$ and the signal $w^*$ is measured by $\left \la \frac {w^{(t)}}{\norm{w^{(t)}}}, w^*\right\ra$.

As in previous work, we also assume that the correlation is positive at initialization. 
\begin{assumption}\label{ass:init} At initialization,  $m_0>0$.
\end{assumption}

Furthermore, we will need the following assumption to control the population gradient (see the proof of Lemma \ref{lem:evol_mt}).

 \begin{assumption}\label{ass:coeff} Let $xf(x)=\sum_{k\geq k^*-1}c_k \frac{H_k(x)}{\sqrt{k!}}$ and $\sigma'(x) = \sum_{k\geq 0}b_k \frac{H_k}{\sqrt{k!}}$ be the Hermite basis decomposition of $xf(x)$ and $\sigma'(x)$. Assume that there are constants $\gamma'>0$, $c>0$  such that for all $x\in [0,\gamma']$, \[ \sum_{k\geq k^*-1}b_kc_kx^{k} \leq -cx^{k^*-1}.\] 
 \end{assumption}

 This assumption allows us to approximate the Hermite decomposition of the gradient by its first non-zero term. A similar assumption was used in \cite{Arous2020OnlineSG} and \cite{anisoSIM}. The relation between the information exponent of $f$ and $x\to xf(x)$ is derived in Proposition \ref{app:lem_ie}.
 
 \paragraph{Example of link function.}Recall that $b_0=0.5$, $b_{2m}=0$ and $b_{2m+1}=\frac{(-1)^m}{\sqrt{2\pi}m!2^m(2m+1)}$ \citep{damian22a}. If $k^*$ is even, we can choose $f^*=\frac{H_{k^*}(x)}{\sqrt{k^*!}}$. The proof of Lemma \ref{app:lem_ie} shows that $xf^*(x)=\frac{H_{k^*+1}(x)}{\sqrt{k^*!}}+\frac{kH_{k^*-1}(x)}{\sqrt{(k^*)!}}$. Consequently, there are only two non-zero coefficients in the Hermite decomposition of $xf(x)$, corresponding to odd Hermite polynomial. Hence Assumption \ref{ass:coeff} is satisfied for small enough $x$, up to a sign factor. 

\begin{remark} One could possibly remove Assumptions \ref{ass:coeff} and \ref{ass:init} by considering a two-layer neural network. Since the assumption on initialization is satisfied with probability $1/2$, it will be satisfied by a constant proportion of the neurons. Moreover, by training the second layer, we could approximate $f$, hence controlling the sign of the coefficients appearing in Assumption \ref{ass:coeff}. However, to show that Assumption \ref{ass:coeff} is satisfied in the two-layer neural network setting, previous work \citep{bietti2022learning, lee2024neural} rely on specific (randomized) link functions, while our analysis relies crucially on the homogeneity of the ReLU activation function. 
\end{remark}

\subsection{Upper-bound on the required sample complexity}

We analyze the upper bound on the sample complexity required to recover the single-index $w^*$.
To simplify its statement, we introduce the following notation for a ratio:
\begin{align}
    \Theta := \Theta(Q,w^*) := \frac{\norm{Q^{1/2}w^*}}{\norm{Q^{1/2}}_F} .
\end{align}
 \begin{theorem}\label{thm:main} Assume that Assumptions \ref{ass:q}, \ref{ass:ie}, \ref{ass:init} and \ref{ass:coeff} hold and the initialization scaling $r$ is such that $\norm{Q^{1/2}w^{(0)}}=c_r\norm{Q^{1/2}w^*}$ for some constant $c_r\in (0,1]$. 
 \begin{enumerate}[noitemsep, topsep=-\parskip, label=(\arabic*)]
     \item When $k^*\geq 3$, choose $\eta=\epsilon_d^2 m_0^{k^*-2} \Theta^2$ where $\epsilon_d \to  0$ as $d\to \infty$.
 Then, after $T= \epsilon_d^{-2} m_0^{2(2-k^*)}\Theta^{-2}$ iterations, $w^{(T)}$ weakly recovers $w^*$, i.e. with probability $1-o(1)$, $m_T\geq \delta $ for some constant $\delta>0$.
 \item If $k^*=1$, the same result holds with the choices $\eta= \epsilon_d^{2} \Theta^2$ and $T= \epsilon_d^{-2} \Theta^{-2}$.
 \item  If $k^*=2$, the result holds with the choices $\eta= \epsilon_d^2 (\log m_0 )^{-1} \Theta^2$ and $T= \epsilon_d^{-2} \log^2 (m_0)\Theta^{-2}$.
 \end{enumerate}
 \end{theorem}

This theorem gives conditions on the sample complexity and the learning rate to ensure that vanilla SGD weakly recovers $w^*$. The proof of this theorem can be divided into two parts. First, we analyze the population dynamic in Section \ref{sec:pop_dyn}. Then, we control the effect of the noise in Section \ref{sec:emp_dyn}. We discuss the extension to strong consistency, i.e., obtaining $m_t\to 1$ in Section \ref{app:descent_phase} in the appendix. 

 \begin{remark}
     The typical order of magnitude of $m_0$ is $\upper$, as shown by concentration inequalities. See Section \ref{app:init} in the appendix. In particular, when $Q=I_d$, $m_0\approx \sqrt{d}^{-1}$ and our upper bounds matches the one obtained by \cite{Arous2020OnlineSG}. When the covariance matrix is aligned with $w^*$, one can have $\sqrt{d}^{-1}\ll m_0$. This accelerates the convergence, as experimentally shown in Section \ref{sec:xp}.
 \end{remark}

\begin{remark}Our bound is, in general, weaker than the one obtained by \cite{anisoSIM} that is of order $dm_0^{(2-2k^*)}$. However, their analysis is based on gradient flow and hence cannot be implemented directly, whereas we analyzed a simple and popular algorithm used in practice. The discretization of the gradient flow is not a straightforward task. As highlighted by our analysis, the choice of the learning rate is crucial and has an effect on the required sample complexity. We believe this is the main reason one can obtain a better bound with gradient flow. Notice, however, that when $Q$ is approximately low-rank, one could have $T\ll d$ whereas the bounds of \cite{anisoSIM} are always at least linear in $d$. 
\end{remark}

\subsection{Correlated Statistical Query (CSQ) lower-bound}

 A common way to provide a lower bound on the required sample complexity for SGD-like algorithms is to rely on the Correlated Statistical Query (CSQ) framework. It is described in Section \ref{app:csq}. 

 \begin{theorem}[CSQ lower-bound]\label{thm:csq} Assume that $\norm{Q^{1/2}}_F^2\gtrsim \norm{Q}_F\sqrt{\log d}$ and $\norm{Q}_F\gtrsim \sqrt{\log d}$. Let us denote \[ v = \min \left( \frac{\norm{Q}_F}{\norm{Q^{1/2}}^2_F}, \frac{1}{\sqrt{d}}\right).\] Then, for any integer $k\geq 1$, there exists a class $\calF_k$ of polynomial functions of degree $k$ such that any CSQ algorithm using a polynomial number of queries $q=O(d^C)$ requires a tolerance of order at most \[ \tau^2 \leq  \epsilon^{k/2}\] where $\epsilon = v\sqrt{\log (qv^{k/2})}$.
 \end{theorem}
 \begin{proof}
     The main difficulty is constructing a large vector family with small correlations measured by the scalar product induced by $Q$. It is detailed in Section \ref{app:csq}. 
 \end{proof}
\begin{remark} The quantity $v$ can be interpreted as the typical value of $m_0$ at initialization. Since one could always use an oracle knowledge of $Q$ to reduce to the isotropic case (where $v=\sqrt{d}^{-1}$), our bound is only useful when $\frac{\norm{Q}_F}{\norm{Q^{1/2}}^2_F}\geq \frac{1}{\sqrt{d}}$. Note that the term $\frac{\norm{Q}_F}{\norm{Q^{1/2}}^2_F}$ corresponds to the average value of $m_0$ when $w^*\sim \calN(0, I_d)$. Hence, our bound is only meaningful for values of $w^*$ close to the average alignment with $Q$. 
\end{remark}
\begin{remark}
    By using the heuristic $\tau = \frac{1}{\sqrt{n}}$ we obtain $n=\Omega(\log d^{k/2} d \left(\frac{\norm{Q}_F}{\norm{Q^{1/2}}^2_F}\right)^{k/2})$. 
    Similarly to previous work \citep{damian22a}, there is a gap in the dependence in $k$ between the upper-bound provided by Theorem \ref{thm:main} and the lower bound. \cite{Smoothing23} show this gap can be removed in the isotropic case by using a smoothing technique. 
\end{remark}

\begin{remark} Theorem \ref{thm:csq} does not cover the cases where $\norm{Q}_F\ll \sqrt{\log d}$ or $\norm{Q^{1/2}}_F^2\gtrsim \norm{Q}_F \sqrt{\log d}$, i.e. the cases where the eigenvalues of $Q$ are quickly decreasing. In these settings, estimating the matrix $Q$ and incorporating it into the algorithm could lead to qualitatively better bounds.     
\end{remark}

\section{Proof Outline of Theorem \ref{thm:main}}
First, we analyze the population dynamic in Section \ref{sec:pop_dyn}. Then, in Section \ref{sec:emp_dyn}, we control the impact of the noise. 

\subsection{Analysis of the Population Dynamics}\label{sec:pop_dyn}

In this section, we assume direct access to the population gradient, meaning that the weights are updated by: \begin{equation}\label{eq:grad_update}
    w^{(t+1)}=w^{t}-\eta \expec_x \nabla_{w^{(t)}}L(y^{(t)}, f_{w^{(t)}}(x^{(t)})).
\end{equation} 
First, we will show that contrary to spherical SGD, the evolution dynamic of $m_t$ also depends on $w^{(t)}$ and $Q$, making the analysis more difficult. More specifically, Lemma \ref{lem:evol_mt} shows that \[  m_{t+1} \approx \frac{\norm{Q^{1/2}w^{(t)}}}{\norm{Q^{1/2}w^{(t+1)}}}m_t+ \eta \frac{ \norm{Q^{1/2}w^*}^2}{\norm{Q^{1/2}w^{(t+1)}}} c m_t^{k^*-1}. \]
In Section \ref{sec:weigth_growth}, we will show that  $\wsig^{(t)}:=\la w^{(t)}, w^*\ra w^*$, the projection of the weight onto the signal component, grows as $m_t$ whereas the growing rate of $\wpe^{(t)}:=w^{(t)}-\wsig^{(t)}$, the projection of the weight onto the component orthogonal to the signal, is slower. As a consequence, as long as $m_t\leq \gamma_1$ for some constant $0<\gamma_1<1$, $\norm{Q^{1/2}w^{(t+1)}}$ remains of the same order than $\norm{Q^{1/2}w^{(0)}}$. By further using the approximation \[ \frac{\norm{Q^{1/2}w^{(t)}}}{\norm{Q^{1/2}w^{(t+1)}}} \approx 1\] that will be formally justified in Section \ref{sec:emp_dyn} (control of $E_1$) we hence obtain the simplified dynamic  \begin{equation}\label{eq:mt_rec}
    m_{t+1}\approx m_t+\tilde{\eta} m_t^{k^*-1}
\end{equation} where $\tilde{\eta}= c\eta \frac{\norm{Q^{1/2}w^*}^2}{2\norm{Q^{1/2}w^{(0)}}}$.

From equation \eqref{eq:mt_rec}, one can show by using Proposition 5.1 in \cite{Arous2020OnlineSG} that, if $m_0>0$, one needs $\frac{k^*-2}{\tilde{\eta} m_0^{k^*-2}}$ iterations to obtain a correlation $m_t$ of constant order when $k^{*}\geq 3$ (the other cases can be treated separately). This result can also be derived heuristically by solving the associated differential equation $f'=\tilde{\eta} f^{k^*-1}$. The population analysis suggests that we should use a large learning rate $\tilde{\eta}$ to accelerate the convergence. However, we will see in Section \ref{sec:emp_dyn} that in order to control the noise, $\tilde{\eta}$ should be small enough.

\subsubsection{Evolution dynamic of $m_t$}

By projecting equation \eqref{eq:grad_update} onto $\frac{Qw^*}{\norm{Q^{1/2}w^*}}$ and dividing by $\norm{Q^{1/2}w^{(t+1)}}$ we obtain:  \footnotesize \begin{align}\label{eq:popdyn}
    m_{t+1}&= \frac{\norm{Q^{1/2}w^{(t)}}}{\norm{Q^{1/2}w^{(t+1)}}}m_t- \nonumber \\&\frac{\eta}{\norm{Q^{1/2}w^{(t+1)}}} \expec_x y \sigma'(\la \frac{w^{(t)}}{\norm{Q^{1/2}w^{(t)}}}, x \ra) \la x, \frac{Qw^*}{\norm{Q^{1/2}w^*}}  \ra .
\end{align} \normalsize
Unlike spherical SGD optimization, which constrains the norm of the weights at each iteration, the evolution of $m_t$ here depends crucially on the evolution of $\norm{Q^{1/2}w^{(t)}}$. For instance, if $\norm{Q^{1/2}w^{(t)}}$ becomes too large, then the information provided by the gradient may be lost. Furthermore, in contrast to the isotropic setting, it is not straightforward how to express the expectation of the gradient as a function of $m_t$. Lemma \ref{lem:evol_mt} addresses this point.

 \begin{lemma}\label{lem:evol_mt} Assume that Assumption \ref{ass:coeff} is satisfied. Then, there exists a constant $\gamma_1>0$ such that as long as the sequence $(m_t)_{t}$ is bounded from above by $\gamma_1$, it satisfies the following relation \begin{equation}\label{eq:evolmt}
 m_{t+1} \geq  \frac{\norm{Q^{1/2}w^{(t)}}}{\norm{Q^{1/2}w^{(t+1)}}}m_t+ \eta \frac{ \norm{Q^{1/2}w^*}^2}{\norm{Q^{1/2}w^{(t+1)}}} c m_t^{k^*-1}. \end{equation}
 \end{lemma}
 \begin{proof}
     By writing $Qw^*= \lambda w^*+ \lambda' w^*_\perp$ where $w^*_\perp$ is a unit vector orthogonal to $w^*$, $\lambda = \la Qw^*,w^*\ra= \norm{Q^{1/2}w^*}^2$ and $\lambda'= \sqrt{\norm{Qw^*}^2-\norm{Q^{1/2}w^*}^4}$, we can decompose the population gradient as \small \[ \expec_x y \sigma'\left( \left\la \frac{w^{(t)}}{\norm{Q^{1/2}w^{(t)}}}, x \right\ra \right) \left\la x, \frac{Qw^*}{\norm{Q^{1/2}w^*}}  \right\ra =G_1 + G_2 \]\normalsize where \small \begin{align*}
    G_1 &= \lambda \expec_x y \sigma'\left(\left\la \frac{w^{(t)}}{\norm{Q^{1/2}w^{(t)}}}, x \right\ra\right) \left\la x, \frac{w^*}{\norm{Q^{1/2}w^*}}  \right\ra\\
    G_2 &=  \lambda' \expec_x y \sigma'\left(\left\la \frac{w^{(t)}}{\norm{Q^{1/2}w^{(t)}}}, x \right\ra\right) \left\la x, \frac{w^*_\perp}{\norm{Q^{1/2}w^*}}\right\ra.
\end{align*}\normalsize

\paragraph{Control of $G_1$.}
To simplify the notations, let us write $z^*=\la x, \frac{w^*}{\norm{Q^{1/2}w^*}}\ra$ and $z_t= \la x,  \frac{w^{(t)}}{\norm{Q^{1/2}w^{(t)}}} \ra $. Notice that $z^*, z_t \sim \calN(0,1)$. Recall that $y=f(z^*)$. So we need to evaluate $\expec_{z^*,z_t}z^*f(z^*)\sigma'(z_t)$. By using the Hermite decomposition of these functions and Proposition \ref{prop:herm_cor}, we can see that this expectation depends essentially on the information exponent of $x\to xf(x)$ and the correlation between $z^*$ and $z_t$. But by Proposition \ref{prop:herm_cor} with $n=1$ (see Section \ref{app:hermite} in the appendix) we have $\expec z^*z_t = m_t$.

By using Hermite decomposition, Proposition \ref{prop:herm_cor}, Lemma \ref{app:lem_ie} 
, and  Assumption \ref{ass:coeff}, we obtain \begin{equation}\label{eq:g1}
    G_1 = \lambda \sum_{l\geq k^*-1} c_lb_lm_t^l \leq -\lambda c m_t^{k^*-1}.
\end{equation}  

\paragraph{Control of $G_2$.}  Denote $z_\perp^*= \la x, \frac{w^*_\perp}{\norm{Q^{1/2}w^*_\perp}}\ra$ and notice that $\expec z_t z_\perp^*=q_t$ where \[q_t= \la \frac{w^{(t)}}{\norm{Q^{1/2}w^{(t)}}}, \frac{w^*_\perp}{\norm{Q^{1/2}w^*_\perp}} \ra_Q.\] We can write $z_t=m_t z^*+\sqrt{1-m_t^2}z_t^\perp $, where $z_t^\perp \sim \calN(0,1)$ is independent from $z^*$. Similarly, we can decompose $z^*_\perp= q_t z_t^\perp + \sqrt{1-q_t^2}\xi $ where $\xi\sim \calN(0,1)$ is independent from $z^*$ and $z_t^\perp$.

Let $p_t= \frac{m_t^2}{2(1-m_t^2)}$. By combining Lemma \ref{lem:gauss_int1} and \ref{lem:gauss_int2} (see Section \ref{app:tech}), we obtain 
\begin{align*}
    \expec &f(z^*)\sigma'(m_t z^*+\sqrt{1-m_t^2}z_t^\perp)z^*_\perp =\\& \frac{q_t}{\sqrt{2\pi(2p_t+1)} }\expec_{z\sim \calN(0,1)} f(\frac{z}{\sqrt{2p_t+1}}).
\end{align*}

We can evaluate $\expec f(\frac{z}{\sqrt{2p_t+1}})$ by applying the multiplicative property of Hermite polynomials recalled in Lemma \ref{lem:mult_herm} in the appendix with $\gamma=\sqrt{2p_t+1}^{-1}$.
Notice that for odd $n$, $\la H_n(\gamma x), H_0(x) \ra =0$ and for $n=2m$, we have \[ \la H_n(\gamma x), H_0(x) \ra = (\gamma^2-1)^{m}\frac{(2m)!}{m!}2^{-m} \]
Since by definition $f(x)=\sum_{k\geq k^*}\frac{a_k}{\sqrt{k!}}H_k(x)$ we get \begin{align*}
    &\abs{ \la f(\frac{x}{\sqrt{2c_t+1}}), H_0(x)\ra } \\&= \abs{\sum_{k\geq k^*/2}\frac{a_{2k}}{\sqrt{(2k)!}}(\frac{1}{2p_t+1}-1)^{k}\frac{(2k)!}{k!}2^{-k}} \\
    & \lesssim \sqrt{\abs{\sum a_{2k}^2\frac{(2k)!}{4^k(k!)^2}}}\sqrt{\abs{\sum p_t^{2k}}}  \tag{by Cauchy-Schwartz}\\
    &\lesssim p_t^{k^*/2}\lesssim m_t^{k^*}.
\end{align*}  
Here we used the fact that by Stirling formula $\frac{(2k)!}{4^k(k!)^2}\sim 1$ so the sequence is bounded, and the fact that by definition of $f$, $\sum a_k^2 =O(1)$. This shows that $G_2$ is of order at most $\lambda ' m_t^{k^*}$ and is negligible compared to $G_1$ as long as $m_t$ is small enough, since $\lambda'\leq 1$.
 \end{proof}

The lower bound obtained in Lemma \ref{lem:evol_mt} is only useful when $m_t>0$. This is ensured by Assumption \ref{ass:init}.

In the next two sections, we are going to control the growth of $\norm{w^{(t)}}$. We are going to show that as long as $m_t\leq \gamma_1$ for some constant $\gamma_1>0$, the weights remain bounded and do not evolve quickly, i.e. $\norm{Q^{1/2}w^{(t)}}\approx \norm{Q^{1/2}w^{(0)}}$  so that equation \eqref{eq:evolmt} is equivalent to \begin{equation}
    \label{eqevolmt2}
     m_{t+1} \geq  m_t+ \tilde{\eta} m_t^{k^*-1}
\end{equation}
where $\tilde{\eta}= c\eta\norm{Q^{1/2}w^*}^2/2\norm{Q^{1/2}w^{(0)}}$. This last relation is similar to the one derived in the isotropic case by \cite{Arous2020OnlineSG}.

\subsection{Control of the growth of $\norm{Q^{1/2}w^{(t)}}$}\label{sec:weigth_growth}

In this section, we justify the approximation $\norm{Q^{1/2}w^{(t)}}\approx \norm{Q^{1/2}w^{(0)}}$ for $t\leq T$.
 By recursion, we obtain \begin{equation}\label{eq:weight_growth}
     Q^{1/2}w^{(t)} = Q^{1/2}w^{(0)}+\eta \sum_{l\leq t-1} \expec y\sigma'(\la w^{(l)}, x \ra ) Q^{1/2}x.
 \end{equation}  Let $w^{(l)}_\perp$ be the projection of $w^{(l)}$ onto the space orthogonal to $w^*$. There are only two directions, $Q^{3/2}w^*$ and $Q^{3/2}w^{(l)}_\perp$, in which the projections of the vector $\expec y\sigma'(\la w^{(l)}, x \ra ) Q^{1/2}x$ are non zero. Indeed, if $v$ is orthogonal to $Q^{3/2}w^*$ then \[ \expec_x \la Q^{1/2}x, v \ra \la x, w^* \ra = \la Q^{1/2}v, w^* \ra_Q = 0\] and similarly for $w^{(l)}_\perp$. As a consequence $\la Q^{1/2}x, v \ra$ is independent of $z$ and $z_t$, and the resulting expectation is zero. To evaluate $\norm{Q^{1/2}w^{(t)}}$ it is sufficient to evaluate the projection of the expectation in the directions identified previously. By a similar analysis as in Lemma \ref{lem:evol_mt} we obtain \[ \expec f(z^*)\sigma'(z_t) \la x, \frac{Q^2w^*}{\norm{Q^{3/2}w^*}}\ra \approx \lambda^2 cm_t^{k^*-1}. \] We can also show that \[ \expec f(z^*)\sigma'(z_t) \la x, \frac{Q^2w^{(l)}_\perp}{\norm{Q^{3/2}w^{(l)}_\perp}}\ra \leq C m_t^{k^*}. \] The details of the calculations can be found in Section \ref{app:weight_growth}. This shows that as long as $\lambda^2c\eta \sum_{t\leq T}m_t^{k^*-1}$ and $C\eta \sum_{t\leq T}m_t^{k^*}$ remains smaller to $0.5\norm{Q^{1/2}w^{(0)}}$, we have \[ 0.5\norm{Q^{1/2}w^{(0)}}\leq  \norm{Q^{1/2}w^{(t)}} \leq 1.5\norm{Q^{1/2}w^{(0)}}. \] The previous conditions are satisfied by choice of the initialization scale, $\eta$ and $T$: the contribution in the direction $Q^{3/2}w^*$ grows slower than $m_{t+1}$, and similarly for the contribution in the other direction.

\subsection{Analysis of the noisy dynamic}\label{sec:emp_dyn}

In this section, we will describe how the noise can be controlled so that after $T$ iterations, $m_t$ is well predicted by the population dynamic analysis performed in Section \ref{sec:pop_dyn}.

We decompose the gradient into two components: the population version and the stochastic noise $V_t$ \[ \nabla_{w^{(t)}}L = \expec(\nabla_{w^{(t)}}L)+V_t. \] 

\subsubsection{Control of $\norm{Q^{1/2}w^{(t)}}$}
As shown in the analysis of the population dynamic, it is critical to control $\norm{Q^{1/2}w^{(t)}}$. In the noisy setting, we obtain the following counterpart of equation \eqref{eq:weight_growth} \begin{align}\label{eq:noisy_weight_growth}
   Q^{1/2}w^{(t)}&= Q^{1/2}w^{(0)}+ \eta \sum_{l\leq t-1} \expec y\sigma'(\la w^{(l)}, x \ra ) Q^{1/2}x \notag\\
   &+\eta Q^{1/2}\sum_{t\leq T}V_t. 
\end{align}
By using Doob's maximal inequality (see Lemma \ref{app:lem_mart2} in appendix), we can show that $\eta \norm{\sum_{t\leq T}V_t}=o(1) $. Hence, the result follows the population dynamic analysis.

\subsubsection{Evolution of $m_t$}

Instead of controlling directly $m_t$, it is more convenient to study the dynamic of the related quantity $\tilde{m}_t:=\la w^{(t)}, \frac{w^*}{\norm{Q^{1/2}w^*}}\ra_Q$ that avoids dividing by the random quantity $\norm{Q^{1/2}w^{t}}$. Since for $t\leq T$ we have $0.5\norm{Q^{1/2}w^{0}}\leq \norm{Q^{1/2}w^{t}}\leq  1.5\norm{Q^{1/2}w^{0}}$, one can easily relate $\tilde{m}_t$ to $m_t$. By definition, we have \begin{align*}
    \tilde{m}_{t+1}&= \tilde{m}_{t}-\eta \expec y\sigma'(\la w^{(t)}, w^* \ra )\la x, \frac{Qw^*}{\norm{Q^{1/2}w^*}}\ra \\&+ \eta \la V_t,  \frac{Qw^*}{\norm{Q^{1/2}w^*}}\ra.
\end{align*}
The expectation term corresponds to $G_1$ analyzed in Lemma \ref{lem:evol_mt} and the stochastic term forms a martingale that Doob's Lemma can control, see Lemma \ref{app:lem_mart1} in the appendix.
Hence, we have obtained a recursion of the form \[ \tilde{m}_{T+1} \geq \eta' \sum_{t\leq T}\tilde{m}_t^{k^*-1} +\eta H_{T} \] where $H_T=\sum_{t\leq T}\la V_t,  \frac{Qw^*}{\norm{Q^{1/2}w^*}}\ra$ and $\eta'=c' \norm{Q^{1/2}w^{(0)}}^{k^*-1}\lambda \eta $. We used the fact that $m_t\geq \frac{2}{3}\norm{Q^{1/2}w^{(0)}} \tilde{m}_t$ for $t\leq T$.

To conclude the proof of Theorem \ref{thm:main}, it remains to understand how many iterations $T$ are necessary so that $m_t$ becomes of constant order. Sequence satisfying $c_{t+1}\geq c_t+\eta m_t^{l}$ have been analyzed formally in \cite{Arous2020OnlineSG} based on Bihari–LaSalle inequality. Here, we present a heuristic way to recover the result.  The continuous analogous of the relation $c_{t+1}\geq c_t+\eta m_t^{l}$ is $f'(t)=\eta f^{l}(t)$. By integrating between $0$ and $T$, we obtain $\frac{1}{f^{l-1}(0)}-\frac{1}{f^{l-1}(T)}=\eta T$ for $l\geq 2$. Since $f(T)$ should be of constant order, it is negligible compared to $\frac{1}{f^{l-1}(0)}=m_0^{-l+1}$. Given the choice $\eta =  \frac{\epsilon}{\sqrt{T}\norm{Q^{1/2}}_F}$ (necessary to control the stochastic error), solving the equation leads to $T=\norm{Q^{1/2}}^2_F\epsilon^{-2}m_0^{-2(l-1)} $.


\section{Numerical Experiments}\label{sec:xp}
In this section, we illustrate our theoretical results through numerical simulations \footnote{The code is available at \url{https://glmbraun.github.io/AniSIM}}. The implementation details are provided in Section \ref{app:xp}.

\paragraph{Anisotropy can help. } We consider the following setting: $y= H_2(\la x, \frac{w^*}{\norm{Q^{1/2}w^*}}\ra)$ where $H_2(x)= x^2-1$, $w^*\in \mS^{d-1}$, and $x\sim \calN(0,Q)$ with a covariance matrix of the form $Q= \frac{I_d+\kappa w^* (w^*)^\top}{1+\kappa}$, parametrized by $\kappa>0$.
The information exponent of $H_2$ is $2$. We set the dimension $d=1000$, the sample size $T=40000$, and the learning rate $0.00002$. The learning dynamics when $\kappa=0$ (isotropic case) is plotted in Figure \ref{fig:xp1subfig1}, while those for $\kappa=6$ in Figure \ref{fig:xp1subfig2}. The improved alignment at initialization when $\kappa=6$ significantly accelerates learning.

\begin{figure}[!ht]
\centering
\begin{subfigure}{0.49\textwidth}
    \includegraphics[width=\textwidth]{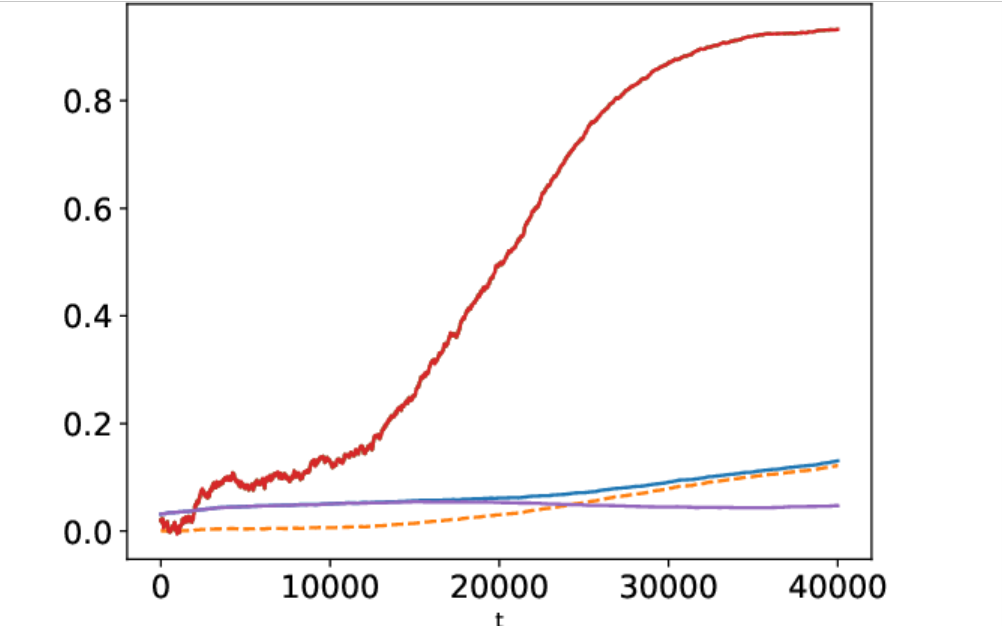}
    \caption{Isotropic setting $\kappa=0$}
    \label{fig:xp1subfig1}
\end{subfigure}
\hfill
\begin{subfigure}{0.49\textwidth}
    \includegraphics[width=\textwidth]{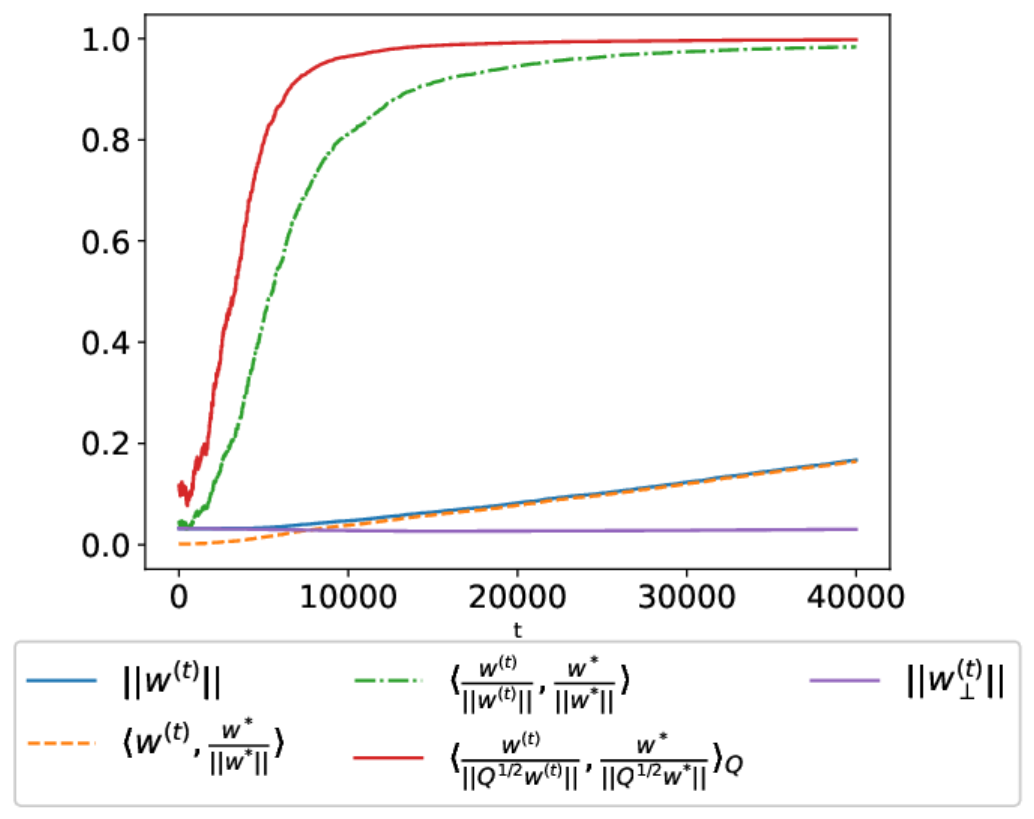}
    \caption{Anisotropic setting $\kappa = 6$}
    \label{fig:xp1subfig2}
\end{subfigure}

\caption{ Comparison of learning dynamics in isotropic and anisotropic settings.
}
\label{fig:xp1}
\end{figure}

\paragraph{Comparison between SGD and Spherical SGD.} We compare the performance of vanilla SGD (\texttt{SGD}) with spherical SGD (\texttt{SpheSGD}). We used an oracle knowledge of the covariance matrix to implement the algorithm. Figure \ref{fig:xp2} shows that the two \texttt{SGD} algorithms behave similarly.

\begin{figure}
    \centering
    \includegraphics[width=1\linewidth]{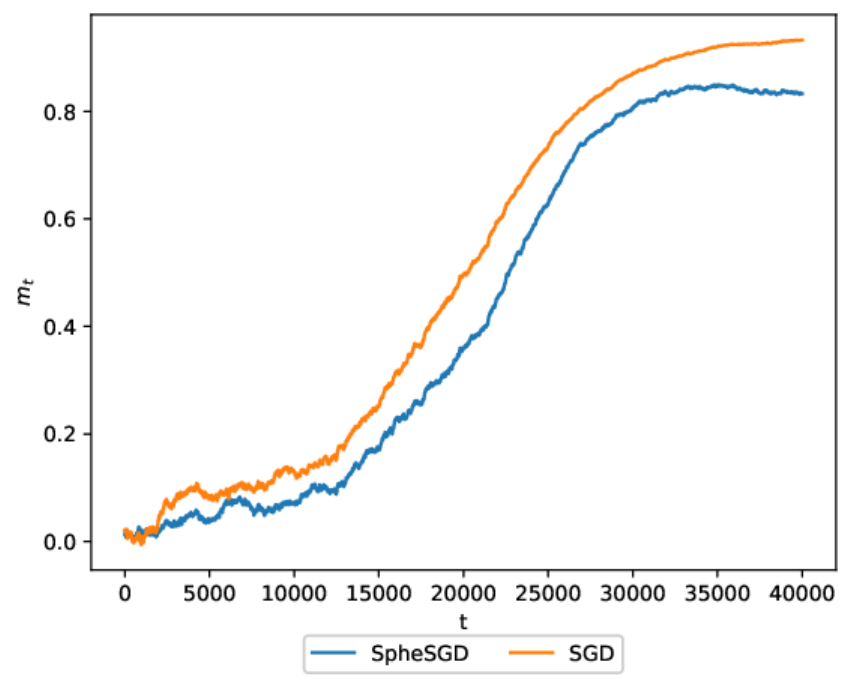}
    \caption{Comparison of learning dynamics between vanilla SGD and spherical SGD.}
    \label{fig:xp2}
\end{figure}

\paragraph{Adaptive Learning Rate.} As shown in Section \ref{app:xp}, progressively increasing the learning rate is beneficial to accelerate the learning dynamic. This is consistent with the theoretical insight that the learning rate should be small enough to control noise at each iteration; however, as the signal increases during training, higher noise can be tolerated.

\paragraph{Batch Reuse.} We demonstrate in Appendix \ref{app:xp} that reusing the same batch can significantly reduce the required sample complexity.


\section{Conclusion}
We analyzed the problem of learning a SIM from anisotropic Gaussian inputs using vanilla SGD. Unlike previous approaches relying on spherical SGD, which require prior knowledge of the covariance structure, our analysis shows that vanilla SGD can naturally adapt to the anisotropic geometry without estimating the covariance matrix. Our theoretical contributions include an upper bound on the sample complexity and a CSQ lower bound, depending on the covariance matrix structure instead of the input data dimension.
Numerical simulations validated these theoretical findings and demonstrated the practical effectiveness of vanilla SGD.

This work opens up several avenues for future research. First, our analysis has focused on the training dynamics of a single neuron, but extending these insights to deeper or wider neural networks would be valuable for a broader understanding of how vanilla SGD performs in more complex architectures. Additionally, achieving the generative exponent in sample complexity by reusing data remains an open question.

\newpage
\bibliography{references}

\section*{Checklist}
 \begin{enumerate}

 \item For all models and algorithms presented, check if you include:
 \begin{enumerate}
   \item A clear description of the mathematical setting, assumptions, algorithm, and/or model. [Yes]
   \item An analysis of the properties and complexity (time, space, sample size) of any algorithm. [Yes]
   \item (Optional) Anonymized source code, with specification of all dependencies, including external libraries. [Yes]
 \end{enumerate}

 \item For any theoretical claim, check if you include:
 \begin{enumerate}
   \item Statements of the full set of assumptions of all theoretical results. [Yes]
   \item Complete proofs of all theoretical results. [Yes]
   \item Clear explanations of any assumptions. [Yes]     
 \end{enumerate}

 \item For all figures and tables that present empirical results, check if you include:
 \begin{enumerate}
   \item The code, data, and instructions needed to reproduce the main experimental results (either in the supplemental material or as a URL). [Yes]
   \item All the training details (e.g., data splits, hyperparameters, how they were chosen). [Yes]
         \item A clear definition of the specific measure or statistics and error bars (e.g., with respect to the random seed after running experiments multiple times). [Yes]
         \item A description of the computing infrastructure used. (e.g., type of GPUs, internal cluster, or cloud provider). [Yes]
 \end{enumerate}

 \item If you are using existing assets (e.g., code, data, models) or curating/releasing new assets, check if you include:
 \begin{enumerate}
   \item Citations of the creator If your work uses existing assets. [Not Applicable]
   \item The license information of the assets, if applicable. [Not Applicable]
   \item New assets either in the supplemental material or as a URL, if applicable. [Not Applicable]
   \item Information about consent from data providers/curators. [Not Applicable]
   \item Discussion of sensible content if applicable, e.g., personally identifiable information or offensive content. [Not Applicable]
 \end{enumerate}

 \item If you used crowdsourcing or conducted research with human subjects, check if you include:
 \begin{enumerate}
   \item The full text of instructions given to participants and screenshots. [Not Applicable]
   \item Descriptions of potential participant risks, with links to Institutional Review Board (IRB) approvals if applicable. [Not Applicable]
   \item The estimated hourly wage paid to participants and the total amount spent on participant compensation. [Not Applicable]
 \end{enumerate}

 \end{enumerate}

\onecolumn
\appendix
\begin{center}
\Large \textbf{Supplementary Material}
\end{center}

We provide background on Hermite polynomials, including key properties that are crucial to our analysis. Technical lemmas involving the evaluation of Gaussian integrals are collected in Section \ref{app:tech}. In Section \ref{app:csq}, we introduce the CSQ framework and prove Theorem \ref{thm:csq}. In Section \ref{app:proof_main}, we complete the proof of Theorem \ref{thm:main}. Finally, in Section \ref{app:xp}, we present additional numerical experiments along with implementation details.

\section{Hermite polynomials}\label{app:hermite}

Consider the probability space $(\R, \calB(\R), \gamma )$ where $\gamma$ denotes the standard Gaussian measure on $\R$ and let $\calH= L^2(\R, \gamma)$ the associated Hilbert space of squared integrable function with respect to $\gamma$.

We will define Hermite polynomial following the approach of \cite{Nualart_Nualart_2018}. Toward this end, let us define two differential operators. For any $f\in C^1(\R)$ we define the \emph{derivative operator} $Df(x)= f'(x)$ and the \emph{divergence operator } $\delta f(x)=xf(x)-f'(x)$. The following lemma show that the operators $D$ and $\delta$ are adjoint.

\begin{lemma}\label{lem:dual}
    Denote by $C_p^1(\R)$ the space of continuously differentiable functions that grows at most polynomially, i.e., there exists some integer $N \geq 1$ and constant $C >0$ such that $|f'(x)|\leq C(1+\abs{x}^N)$. For any $f,g\in C^1_p(\R)$, we have \[ \langle Df, g\rangle_\calH = \langle f, \delta g\rangle_\calH. \]
\end{lemma}
\begin{proof}
    The result is derived directly by using integration by parts.
\end{proof}

The Hermite polynomials are defined as follows: \begin{align*}
    H_0(x)&= 1,\\
    H_n(x)&= \delta^n 1.
\end{align*}

\begin{proposition} The sequence of normalized Hermite polyniomials $(\frac{1}{\sqrt{n!}}H_n)_{n\geq 0}$ is an orthonormal basis of $\calH$.
\end{proposition}

Another particularly useful property of the Hermite polynomial that we will rely on heavily is the simple characterization of the correlation between two Hermite polynomials with correlated Gaussian inputs.
\begin{proposition}\label{prop:herm_cor}
    Let $x\sim \calN(0, I_d)$. For any $w, w'\in \mS^{d-1}(\R)$ we have \[ \expec_x H_n\left(\langle x, w\rangle\right)H_{n'}\left(\langle x, w'\rangle\right)= \indic_{\lbrace n=n' \rbrace }n!\langle w, w'\rangle^n.\]
\end{proposition}
This result can be extended straightforwardly to anisotropic inputs.
\begin{corollary}\label{cor:herm_cor} Let $x\sim \calN(0, Q)$ for some general covariance matrix $Q$. For any $w, w'\in \R^d$ we have \[ \expec_x H_n\left(\langle x, \frac{w}{\norm{Q^{1/2}w}}\rangle\right)H_{n'}\left(\langle x, \frac{w'}{\norm{Q^{1/2}w'}}\rangle\right)= \indic_{\lbrace n=n' \rbrace }n!\left(\langle \frac{w}{\norm{Q^{1/2}w}}, \frac{w'}{\norm{Q^{1/2}w'}}\rangle_Q\right)^n.\]  
\end{corollary}

The following lemma relates the information exponent of $f$ to $x\to xf(x)$, the function naturally appearing in the population gradient.

\begin{lemma}\label{app:lem_ie}
    Assume that $f\in \calH$ has information exponent $k\geq 1$. Then, the function $x\to xf(x)\in \calH$ has information exponent $k-1$.
\end{lemma}
\begin{proof}
    Assume that $k\geq 2$. Recall that the Hermite's polynomials satisfy $H'_n(x)=nH_{n-1}$ (e.g. it derives easily as an application of Lemma \ref{lem:dual}) and $H_{n+1}(x)=xH_n(x)-H'_n(x)$ (by definition $H_{n+1}=\delta^{n+1}\indic = \delta H_n$). By consequence, for all $n\in \mathbb{N}^*$, we have \begin{align*}
        \langle xf(x) , H_{n}(x) \rangle_\calH &= \la f(x), xH_{n}(x) \ra_\calH \\ 
        &=\la f(x), H_{n+1}(x) \ra_\calH + \la f(x), H_{n}'(x) \ra_\calH\\
        &= \la f(x), H_{n+1}(x) \ra_\calH + n\la f(x), H_{n-1}(x) \ra_\calH.
        \end{align*} If $n< k-1$ all the terms are null. But for $n=k-1$, $\la f(x), H_{n+1}(x) \ra_\calH \neq 0$, while  $n\la f(x), H_{n-1}(x) \ra_\calH =0$ by definition of $k$. The case where $k=1$ can be treated similarly.
\end{proof}


\section{Technical lemmas}\label{app:tech}
Recall that $\sigma'$ is the sign function, formally defined as \[ \sigma'(x)= \begin{cases} 
1 & \text{if } x > 0, \\
0 & \text{if } x = 0, \\
-1 & \text{if } x < 0.
\end{cases}\]

\begin{lemma}\label{lem:gauss_int1}Let $p\in [0,1]$ and $X, Y$ be two independent standard Gaussian r.v. We have \[ \expec_Y \sigma'\left(p X+ \sqrt{1-p^2}Y\right)Y = \frac{1}{\sqrt{2\pi}} e^{-\frac{p^2}{2(1-p^2)}X^2}. \]
\end{lemma}
\begin{proof}
    By symmetry, we have \begin{align*}
        2\expec_Y \sigma'(p X+ \sqrt{1-p^2}Y)Y &=\expec_Y \left(\sigma'(p X+ \sqrt{1-p^2}Y)- \sigma'(p X- \sqrt{1-p^2}Y)\right)Y \\
        & = \expec_Y \indic_{\lbrace |Y|\geq \frac{p}{\sqrt{1-p^2}}|X|\rbrace }|Y|\\
        & = \frac{2}{\sqrt{2\pi}}e^{-\frac{p^2}{2(1-p^2)}X^2}.
    \end{align*}
\end{proof}
\begin{lemma}\label{lem:gauss_int2} Let $c>0$, $X\sim \calN(0,1)$ and $f\in \calH$. We have \[ \expec_X f(X)e^{-cX^2}= \frac{1}{\sqrt{2c+1}}\expec_Xf\left(\frac{X}{\sqrt{2c+1}}\right).\]
\end{lemma}
\begin{proof}
    Use the change of variable $u=\sqrt{1+2c}x$.
\end{proof}

\begin{lemma}\label{lem:mult_herm}
    For every $\gamma>0$, $n \in \mathbb{N}^*$ we have \[ H_n(\gamma x) = \sum_{k=0}^{\frac{n}{2}} \gamma^{n-2k}(\gamma^2-1)^k\binom{n}{2k}\frac{(2k)!}{k!}2^{-k}H_{n-2k}(x).\]
\end{lemma}

\begin{proof}
  This idendity is classical, but since we didn't find a proper reference, we provide a simple proof. Recall that the Hermite polynomials satisfy the following identity for all $t, x \in \R $ (see \cite{Nualart_Nualart_2018}) \[ e^{-\frac{t^2}{2}+tx} = \sum_{n\geq 0}H_n(x)\frac{t^n}{n!}.\] So we have \[ \sum_{n\geq 0}H_n(\gamma x)\frac{t^n}{n!} =e^{-\frac{t^2}{2}+\gamma tx}=e^{-\frac{\gamma^2t^2}{2}+\gamma tx}e^{\frac{(\gamma^2-1)t^2}{2}}.\] By using the series development of the previous exponential functions we get \begin{align*}
      \sum_{n\geq 0}H_n(\gamma x)\frac{t^n}{n!} &= \sum_{j\geq 0}H_j(x)\frac{(\gamma t)^j}{j!} \sum_{k\geq 0}\frac{(\gamma^2-1)^kt^{2k}}{2^kk!}\\
      &= \sum_j \sum_k \frac{t^{j+2k}}{(j+2k)!}H_j(x)\gamma^j(\gamma^2-1)^k2^{-k}\frac{(j+2k)!}{j!k!}. 
  \end{align*}
  Let $n=j+2k$. By identifying the coefficient associated in the serie expansion, we obtain the stated formula. 
\end{proof}
\section{CSQ lower-bound}\label{app:csq}
The Correlational Statistic Query framework is a restricted computational model where we access knowledge of the data distribution $(x,y)\sim P$ by addressing query $\phi: \R^d\to \R$ to an oracle that returns $\expec_{(x,y)\sim P}(y\phi(x))+\epsilon$ where $\epsilon$ is some noise term bounded by $\tau$, the tolerance parameter. SGD is an algorithm belonging to this framework (note, however, that in the CSQ framework, the noise can be adversarial). 

The classical way to obtain a lower bound \citep{damian22a} is to construct a large class of function $\calF$ with small correlations. The following lemma provides a lower bound.

\begin{lemma}[\cite{csq09}, \cite{damian22a}] \label{lem:szorenyi}Let $\calF$ be a class of function and $\calD$ be a data distribution such that \[ \expec_{x\sim \calD} f^2(x), \quad |\expec_{x\sim \calD}f(x)g(x)|\leq \epsilon, \quad \forall f\neq g\in \calF.\] Then any CSQ algorithm requires at least $\frac{|\calF|(\tau^2-\epsilon)}{2}$ queries of tolerance $\tau$ to output a function in $\calF$ with $L^2(\calD)$ loss at most $2-2\epsilon$.
    
\end{lemma}

We then usually use the heuristic $\tau= \frac{1}{\sqrt{n}}$ to derive a lower bound on the sample complexity.  


\begin{lemma}\label{lem:csq}
Assume that the covariance matrix \(Q\) satisfies \(\norm{Q} = 1\), \(\norm{Q^{1/2}}_F^2 \gtrsim \norm{Q}_F \sqrt{\log d}\), and \(\norm{Q}_F \geq C \sqrt{\log d}\) for some sufficiently large constant \(C > 0\). Let 
\[
m = C\frac{\norm{Q}_F}{\norm{Q^{1/2}}_F^2}.
\]
Then, for \(\epsilon = m \sqrt{\log (qm^{k/2})}\), where \(q = O(d^{c})\) (for some constant \(c > 0\)) is the number of queries, there exists an absolute constant $C_1$ and a set \(\calE \)  with cardinality at least \(0.5e^{C_1\log (qm^{k/2})}\) such that $\forall w\neq v \in \calE$ we have \[
\abs{\left\langle \frac{w}{\norm{Q^{1/2}w}}, \frac{v}{\norm{Q^{1/2}v}} \right\rangle_Q} \leq \epsilon
.\]

\end{lemma}
\begin{proof} This is an adaptation of Lemma 3 in \cite{damian22a} to the anisotropic case.
Let \(w_1, \ldots, w_p\) be i.i.d. Gaussian random variables \(w_i \sim \calN(0, I_d)\). By the Hanson-Wright inequality, for all \(i \in [p]\), we have:
\[
\prob\left( \abs{w_i^\top Q w_i - \norm{Q^{1/2}}_F^2} \geq t \right) 
\leq e^{-c \min \left( \frac{t^2}{\norm{Q}_F^2}, \frac{t}{\norm{Q}} \right)}.
\]
Since \(w_i^\top Q w_i = \norm{Q^{1/2}w_i}^2\), choosing \(t = C \norm{Q}_F \sqrt{\log(qm^{k/2})}\), we get that this probability is bounded by \(e^{-c' \log(qm^{k/2})} \cup e^{-C\log d}\), where \(c' = C / c\). This holds because, by assumption, \(\norm{Q}_F \gtrsim \sqrt{\log(d)}\) and $\log (qm^{k/2})\gtrsim \log d $.

Similarly, for every \(i \neq j \in [p]\), we have:
\[
\prob\left( \abs{w_i^\top Q w_j} \geq t \right) 
\leq e^{-c \min \left( \frac{t^2}{\norm{Q}_F^2}, \frac{t}{\norm{Q}} \right)}
\leq e^{-c' \log(qm^{k/2})}.
\]

Using a union bound over all pairs \(i, j \in [p]\), we obtain that, with probability at least \(1 - 2p^2 e^{-c' \log(qm^{k/2})}\):
\[
\forall i \in [p], \quad c_1 \norm{Q^{1/2}}_F^2 \leq \norm{Q^{1/2}w_i}^2,
\]
and for all \(i \neq j \in [p]\):
\[
\abs{\langle w_i, w_j \rangle_Q} \leq C \norm{Q}_F \sqrt{\log(qm^{k/2})}.
\]

This implies that for \(i \neq j\):
\[
\abs{\left\langle \frac{w_i}{\norm{Q^{1/2}w_i}}, \frac{w_j}{\norm{Q^{1/2}w_j}} \right\rangle_Q} \leq \epsilon,
\]
where \(\epsilon = C \frac{\norm{Q}_F}{\norm{Q^{1/2}}_F^2} \sqrt{\log(qm^{k/2})}\), completing the proof.
\end{proof}

\begin{theorem} Assume that the assumptions of Lemma \ref{lem:csq} are satisfied. For any integer $k\geq 1$, there exists a class $\calF_k$ of polynomial functions of degree $k$ such that any CSQ algorithm using a polynomial number of queries $q$ requires a tolerance $\tau$ of order at most \[ \tau^2 \leq \epsilon^{k/2}. \]  
\end{theorem}
\begin{remark}
    By using the heuristic $\tau^2 = \frac{1}{\sqrt{n}}$ we obtain $n=\Omega(\log d^{k/2} d \left(\frac{\norm{Q}_F}{\norm{Q^{1/2}}^2_F}\right)^{k/2})$. The term $\frac{\norm{Q}_F}{\norm{Q^{1/2}}^2_F}$ corresponds to the average value of $m_0$ when $w^*\sim \calN(0, I_d)$. Similar to previous work \citep{damian22a}, there is a gap in the dependence in $k$ between the upper-bound provided by Theorem \ref{thm:main} and the lower bound. \cite{Smoothing23} show this gap can be removed in the isotropic case by using a smoothing technique.  
\end{remark}
\begin{remark} Lemma \ref{lem:csq} doesn't cover the cases where $\norm{Q}_F\ll \sqrt{\log d}$ or $\norm{Q^{1/2}}_F^2\gtrsim \norm{Q}_F \sqrt{\log d}$, i.e. the cases where the eigenvalues of $Q$ are quickly decreasing. In these settings, estimating the matrix $Q$ and incorporating it into the algorithm could lead to qualitatively better bounds.     
\end{remark}
\begin{proof}
    Recall that $\calE$ is the set constructed in Lemma \ref{lem:csq} and consider the class of functions \[ \calF_k =\left\lbrace x\to \frac{H_k\left( \left\la x, \frac{w}{\norm{Q^{1/2}w}}\right\ra\right)}{\sqrt{k!}} \,\middle|\, w\in \calE\right\rbrace.\]
    For any $w\neq w'\in \calE$ we have \[ \abs{\left\la  \frac{H_k\left(\la x, \frac{w}{\norm{Q^{1/2}w}}\ra\right)}{\sqrt{k!}}  ,  \frac{H_k\left(\la x, \frac{w'}{\norm{Q^{1/2}w'}}\ra\right)}{\sqrt{k!}} \right\ra_Q}\leq \epsilon ^k.\]
    We obtain the result from Lemma \ref{lem:szorenyi} and elementary algebra.
\end{proof}

\section{Additional proofs}
\subsection{Proof of the claims in Section \ref{sec:weigth_growth}}\label{app:weight_growth}
Recall the decomposition $Qw^*=\lambda w^*+\lambda' w^*_\perp$ of Lemma \ref{lem:evol_mt}. We have  $Q^2w^*= Q(\lambda w^*+\lambda'w_\perp^*)= \lambda^2w^*+ \lambda \lambda' w^*_\perp +\lambda'Qw_\perp^*$. Since $\la Qw_\perp ^* , w^* \ra = \la w_\perp ^* , Qw^* \ra=\lambda'$, we obtain \[ Q^2w^* = (\lambda^2 +(\lambda')^2)w^*+ \lambda \lambda' w_\perp^* + \lambda' \lambda'' \tilde{w}_\perp^* \] where $\lambda''\tilde{w}_\perp^*= Qw_\perp^*-\lambda' w^*$. As a consequence, we can decompose \[ \expec f(z^*)\sigma'(z_t)\la x, \frac{Q^2w^*}{\norm{Q^{3/2}w^*}} \ra = G'_1+ G'_2 + G'_3 \] where \begin{align*}
    G'_1 &= (\lambda^2+ (\lambda')^2)\frac{\norm{Q^{1/2}w^*}}{\norm{Q^{3/2}w^*}} \expec z^*f(z^*)\sigma'(z_t)\\
    G'_2 &= \lambda \lambda'\frac{\norm{Q^{1/2}w_\perp^*}}{\norm{Q^{3/2}w^*}} \expec f(z^*)\sigma'(z_t)z^*_\perp\\
    G'_3 &= \lambda '\lambda''\frac{\norm{Q^{1/2}\tilde{w}_\perp^*}}{\norm{Q^{3/2}w^*}} \expec f(z^*)\sigma'(z_t)\tilde{z}^*_\perp
\end{align*} where $\tilde{z}^*_\perp=\la x, \frac{Q^{1/2}\tilde{w}_\perp^*}{\norm{Q^{1/2}\tilde{w}_\perp^*}}\ra $.  Notice that $G_1'=(\lambda^2+ (\lambda')^2)\frac{\norm{Q^{1/2}w^*}}{\norm{Q^{3/2}w^*}}G_1 \approx -c\lambda(\lambda^2+ (\lambda')^2)\frac{\norm{Q^{1/2}w_\perp^*}}{\norm{Q^{3/2}w^*}}m_t^{k^*-1}$ by the proof of Lemma \ref{lem:evol_mt}. The terms $G'_2$ and $G'_3$ can also be analyzed as in Lemma \ref{lem:evol_mt}. 

\section{Proof of Theorem \ref{thm:main}}\label{app:proof_main}
In this section, we complete the proof of Theorem \ref{thm:main} sketched in the main text.

\subsection{Initialization}\label{app:init}

Here, we justify the claim that $m_0$ is of order  $\upper$ with positive probability. 

Recall that $w'\sim \calN(0,I_d)$. Hence, $\la w' , Qw^* \ra \sim \calN(0, \norm{Qw^*}^2)$. This implies that $m_0>0$ with probability $1/2$ and \[ \prob (c_1\sigma \leq \abs{\la w' , Qw^* \ra} \leq c_2\sigma)= 1-\prob( \abs{\la w' , Qw^* \ra} \geq c_2\sigma) - \prob( \abs{\la w' , Qw^* \ra} \leq c_1\sigma).\] But \[ \prob( \abs{\la w' , Qw^* \ra} \geq c_2\sigma) \leq e^{-c_2^2/2}\] and \[ \prob( \abs{\la w' , Qw^* \ra} \leq c_1\sigma) \leq 1-e^{-c_1^2/2}.\] So, if $c_1$ is chosen small enough, and $c_2$ large enough \[ \prob (c_1\sigma \leq \abs{\la w' , Qw^* \ra} \leq c_2\sigma) \geq 1-\epsilon.\]

Now, let us control $\norm{Q^{1/2}w'}^2= w'^\top Qw'$. This is a quadratic form in $w'$ that has expectation $\expec \norm{Q^{1/2}w'}^2 = \norm{Q^{1/2}}_H^2$. By Hanson-Wright inequality, we have \[ \prob(\abs{\norm{Q^{1/2}w'}^2-\norm{Q^{1/2}}_H^2}\geq t)\leq e^{-c\min(\frac{t^2}{\norm{Q}_H^2}, \frac{t}{\norm{Q}})}.\] By choosing $t=C\norm{Q}_H\lesssim \norm{Q^{1/2}}_H^2  $ since $\norm{Q}=1$, we obtain that with positive probability \[0.5 \norm{Q^{1/2}}_H^2 \leq \norm{Q^{1/2}w'}^2 \leq 1.5 \norm{Q^{1/2}}_H^2.\] We obtaine the claimed result by taking the quotient.

    



\subsection{Control of the noise}
First, let us recall Doob's maximal inequality that will be used frequently to control the noise.

\begin{theorem}[Doob's Maximal Inequality]\label{app:doob_thm}
    Let $(X_t)_{t\leq T}$ be a martingale or positive submartingale belonging to $L^p$ for some $p\geq 1$. Then for every $\lambda >0$ we have \[ \prob \left( \sup_{t\leq T}\abs{X_t}\geq \lambda \right) \leq \frac{\expec \abs{X_T}^p}{\lambda^p}.\]
\end{theorem}

\begin{lemma}\label{app:lem_mart1}
    For all $\epsilon>0$ we have \[ \prob\left( \sup_{t\leq T}\abs{\sum_{l\leq t} \la V_l, w^*\ra}\geq  \frac{\sqrt{T}}{\epsilon}\right) \leq \epsilon^2 \norm{Q^{1/2}w^*}^2. \]
\end{lemma}
\begin{proof}
    Let us define $M_t=\sum_{l\leq t} \la V_l, w^* \ra$. Notice that since $y\sigma'(.)$ is always bounded by one, we have \[ \expec (M_T^2)\leq 
 T \expec \la x, w^*\ra ^2 \leq \norm{Q^{1/2}w^*}^2T.\]
 By consequence, Theorem \ref{app:doob_thm} applied with $p=2$ leads to the result. Notice that the bound is uniform is the initial value $w^{(0)}$, as in \cite{Arous2020OnlineSG}.
\end{proof}


\begin{lemma}\label{app:lem_mart2} Let us denote $M'_t= \norm{\sum_{l\leq t}V_l}^2$. This is a submartingale and for all $\epsilon >0$ we have \[ \prob\left(\sup_{t\leq T}M'_t\geq \frac{\norm{Q^{1/2}}_F^2}{\epsilon} T\right)\leq \epsilon .\] In particular, it implies that with probability at least $1-\epsilon $, for all $t\leq T$, we have \[ \norm{\sum_{l\leq t}V_l}\leq \epsilon^{-1/2}\sqrt{T}\norm{Q^{1/2}}_F.\]
\end{lemma}

\begin{proof}
    By definition $M'_t= \sum_i  \la \sum_l V_l, e_i \ra^2$. Since any convex function of a martingale is a submartingale, $\la \sum_l V_l, e_i \ra^2$ is a submartingale and $M'_t$ is a submartingale as a sum of submartingale. 

    Now observe that \begin{align*}
        \expec M'_T &= \expec \sum_{t\leq T} \norm{V_t}^2 + \expec \sum_{t\neq t'}\la V_t, V_{t'} \ra \\
        &=  \expec \sum_{t\leq T} \norm{V_t}^2 \tag{since $\expec(V_{t+1}|H_t)=0$.}\\
        &\leq T \norm{Q^{1/2}}_H^2.
    \end{align*}
\end{proof}

\subsection{The Descent Phase}\label{app:descent_phase}
 Assume that Assumption \ref{ass:coeff} is valid for $\gamma'=1$.

It is clear from the previous analysis in Section \ref{sec:emp_dyn} that the directional martingale error term $E_2$ is negligible compared to $m_t$. However, $\norm{Q^{1/2}w^{(t+1)}}$ is no longer necessarily bounded and the approximate dynamic \eqref{eq:mt_rec} is no longer valid. The analysis done in section \ref{sec:pop_dyn} suggests that $\norm{Q^{1/2}w^{(t)}}$ grows at a similar rate than $m_t$ and if the initial scaling of the weights $r$ is small enough, one should have $\norm{Q^{1/2}w^{(t)}}\approx m_t$. Hence, from Lemma \ref{lem:evol_mt} we obtain the following approximated dynamic \[ m_{t+1}\approx \frac{m_t}{m_{t+1}}m_t+\frac{\eta}{m_{t+1}}m_{t}^{k^*-1}\] that is equivalent to \[ m_{t+1}^2\approx m_t^2+\eta m_{t}^{2\frac{(k^*-1)}{2}}\] that can be solved similarly as \eqref{eq:mt_rec} with the change of variable $u_t=m_t^2$. All these approximations remain to be made rigorous.

\section{Additional numerical experiments} \label{app:xp}
The experiments were conducted using Python on a CPU Intel Core i7-1255U. 
The code is available at \url{https://glmbraun.github.io/AniSIM}.

\subsection{Description of \texttt{SpheSGD}}
The spherical gradient $\nabla^s$ with respect to the geometry induced by $Q$ is defined by \[ \nabla^s_w L = \nabla_w L - \la \nabla_w L, w \ra_Q w.  \] We update the weights as follows \begin{align*}
    \Tilde{w}^{(t+1)} &= w^{(t)}- \eta \nabla^s_{w^{(t)}} L\\
    w^{(t+1)} &= \frac{\Tilde{w}^{(t+1)}}{\norm{Q^{1/2}\Tilde{w}^{(t+1)}}}.
\end{align*}

\subsection{Adaptative learning rate.}
The theoretical analysis shows that $\eta$ should be chosen small enough to control the impact of noise. However, as the signal increases, more noise can be tolerated. Here, we consider a SIM of the form of the form $y=\texttt{Sign} (\la x, w^* \ra )$ where $x\sim \calN(0,I_d)$ with $d=4000$, $n=8000$ and $w^*\in \mathbb{S}^{d-1}$.  We run vanilla SGD with a learning rate $\eta= 0.000001$ and $\texttt{AdaptLR-SGD}$ where at each gradient step the learning rate is increased: $\eta_{t+1}=\eta_t(1+0.000001)$. As shown in Figure \ref{fig:xp4},  increasing the learning rate accelerates the algorithm's convergence. Determining a data-driven method to select an appropriate learning rate is left for future work.

\begin{figure}
    \centering
    \includegraphics{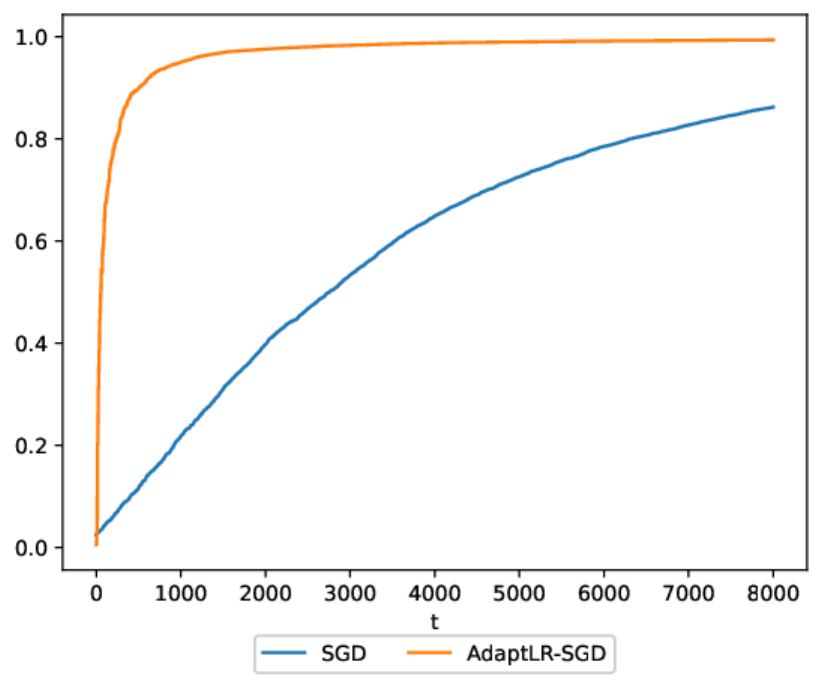}
    \caption{Comparison of learning dynamics between $\texttt{Vanilla SGD}$ and $\texttt{RepSGD}$.}
    \label{fig:xp4}
\end{figure}

\subsection{Data reuse}

We consider the following algorithm referred to as $\texttt{RepSGD}$, that use the same data batch two times: \begin{align*}
    \tilde{w}^{(t)}&=w^{t}-\eta_1\nabla^s_{w^{(t)}} L\\
    w^{(t+1)} &= w^{t}-\eta_2\nabla^s_{\tilde{w}^{(t)}} L.
\end{align*} 
This is similar to the algorithm analyzed in \cite{repetita24}, except that we do not use spherical gradient update nor use a retractation to ensure that the norm of $w^{(t)}$ remains equal to one.

We consider the learning the following single index model $y=H_3(\la w^*, x\ra )$ with $x\sim \calN(0, I_d)$, $d=4000$ and $n=80000$. We fix the learning rates $\eta_1=\eta=\eta_2= -0.0001$. 

Figure \ref{fig:xp3} shows that while vanilla $\texttt{SGD}$ is unable to learn the single index $w^*$, $\texttt{RepSGD}$ achieves weak recovery with the same sample complexity.

\begin{figure}
    \centering
    \includegraphics{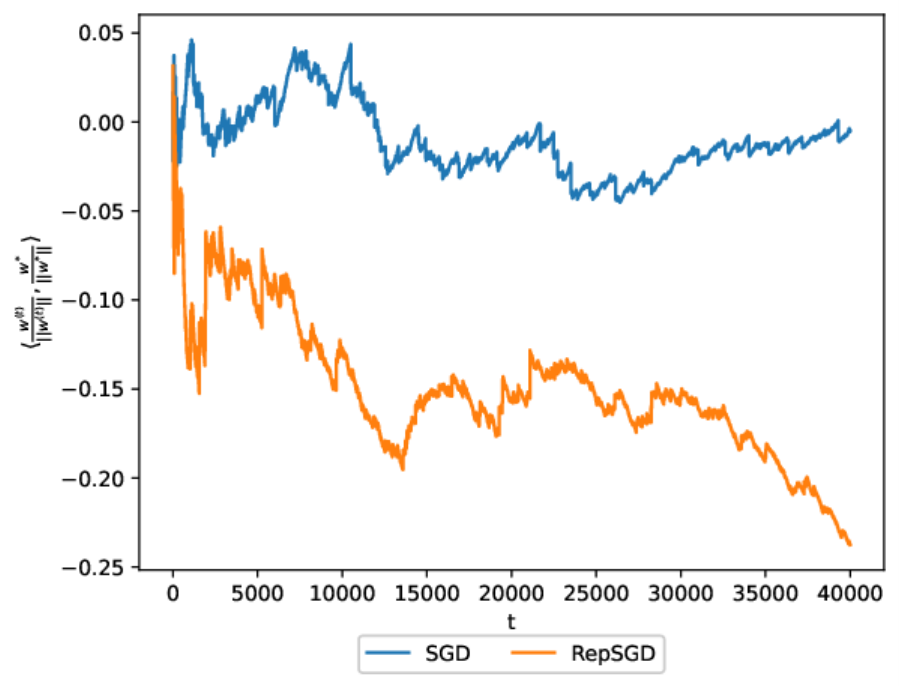}
    \caption{Comparison of learning dynamics between $\texttt{Vanilla SGD}$ and $\texttt{RepSGD}$.}
    \label{fig:xp3}
\end{figure}

\end{document}